\documentclass[12pt, letterpaper]{article}

\usepackage{amsmath,amsthm,amssymb}
\usepackage{listings}
\usepackage{graphicx}
\usepackage{python}
\usepackage[font=scriptsize]{caption}
\usepackage{subcaption}
 \usepackage{here}
\usepackage{color}
\usepackage[letterpaper, margin=1in]{geometry}
\usepackage{setspace}
\usepackage{titling,lipsum}
\usepackage[dvipsnames,svgnames,table]{xcolor}
\usepackage{hyperref}
\usepackage{multicol}
\usepackage{mathtools}
\usepackage{float}
\usepackage{booktabs}
\usepackage{natbib}
\usepackage{mathtools}
\usepackage{preview}
\usepackage{standalone}
\usepackage[capitalize,nameinlink]{cleveref}
\usepackage[toc,page]{appendix}
\usepackage{algorithm}
\usepackage{tikz}
\usepackage{multirow}
\hypersetup{colorlinks,linkcolor={Blue},citecolor={Blue},urlcolor={Blue}}
\newtheorem{thm}{Theorem}

\theoremstyle{definition}

\newcommand{\argmin}{\operatornamewithlimits{argmin}}

\newcommand{\tr}{\operatornamewithlimits{Trace}}

\newcommand{\bz}{\bold{z}}
\newcommand{\by}{\bold{y}}
\newcommand{\bx}{\bold{x}}
\newcommand{\bu}{\bold{u}}

\newcommand{\bX}{\bold{X}}
\newcommand{\bU}{\bold{U}}
\newcommand{\px}{\pmb{x}}

\newcommand{\bbeta}{\boldsymbol{\beta}}
\newcommand{\beps}{\boldsymbol{\varepsilon}}

\newcommand{\bK}{\bold{K}}

\newcommand{\bk}{\bold{k}}

\newcommand{\PPhi}{\boldsymbol{\Phi}}

\title{Extrinsic Kernel Ridge Regression Classifier for Planar Kendall Shape Space}

\author{
  Hwiyoung Lee\\
  \texttt{hwiyoung.lee@stat.fsu.edu}
  \and
  Vic Patrangenaru\\
  \texttt{vic@stat.fsu.edu}
}

\date{Department of Statistics, Florida State University\\ \today}

\begin{document}

\maketitle

\begin{abstract}
\onehalfspacing
\selectfont
Kernel methods have had great success in Statistics and Machine Learning. Despite their growing popularity, however, less effort has been drawn towards developing kernel based classification methods on Riemannian manifolds due to difficulty in dealing with non-Euclidean geometry. In this paper, motivated by the extrinsic framework of manifold-valued data analysis, we propose a new positive definite kernel on planar Kendall shape space $\Sigma_2^k$, called extrinsic Veronese Whitney Gaussian kernel. We show that our approach can be extended to develop Gaussian kernels on any embedded manifold. Furthermore, kernel ridge regression classifier (KRRC) is implemented to address the shape classification problem on $\Sigma_2^k$, and their promising performances are illustrated through the real data analysis. \vspace{.5em}

\end{abstract}

%%%%%%%%%%%%%%%%%%%%%%%%%%%%%%%%%%%%%
%%%%%%%%%%%            1. Introduction              %%%%%%%%%%%
%%%%%%%%%%%%%%%%%%%%%%%%%%%%%%%%%%%%%

\section{Introduction}

Classification has been one of the main research topics for quite a long time in statistics and machine learning literature. A classification problem can be generally formulated as follows. Given a training data set $\{(\mathbf{x}_i,\by_i)\}_{i=1}^n$, where $\mathbf{x}_i \in \mathcal{X}$, $\by_i \in \mathcal{Y}$ and $\mathcal{Y}$ is a finite discrete set, we consider the following model $\by_i = f(\mathbf{x}_i) + \beps_i$. The goal is to construct a classifier $f : \mathcal{X} \rightarrow \mathcal{Y}$ that the class label of a new testing point $\bx \in \mathcal{X}$ can be successfully predicted by the model with good accuracy. As reflected in the number of articles published, many classification methods have been intensively studied on Euclidean space, whereas much less attention has been paid to non Euclidean spaces. In new types of data analysis emerged in recent years, however, analyzing non-Euclidean data, mostly manifold-valued data, has attracted great interest due to their potential applications in a various scientific fields. Examples of such data types include directions of points on a sphere, planar or 3D shapes, diffusion tensor magnetic resonance images (DT-MRI), CT images and other types of medical images. Despite the recent popularity of non-Euclidean data analysis, implementing existing classification methods such as support vector machines, multiclass logistic regression, K-nearest neighbors and their variants, under non-Euclidean settings has remained largely unaddressed. Because of the restriction requiring Euclidean forms placed on predictors, the impracticability of directly applying the aforesaid Euclidean classification methods to such data types is inevitable. For instance, suffering from the difficulty for defining the notion of a centroid and choosing a proper distance on non-Euclidean space, K-nearest neighbors method is not immediately feasible. A number of different statistical methods have emerged in an attempt to counter this difficulty in non-Euclidean data analysis.

In this paper, we adopt a kernelized method as part of efforts to tackle this problem. Specifically, we work on developing a kernel based classifier on Kendall's planar shape space which is one of the popular Riemannian manifolds that has been intensively studied in the object data and shape analysis literatures \citep{WaMa:2007, MaAl:2014, PaEl:2015, DrMa:2016}. While most of the focus in a rich literature on Kendall's planar shape space mainly has been on statistical inference based on various types of metrics defined on this space, there has been relatively scant research conducted with kernels (\citet{JaSaLiHa:2013}, \citet{LiMuChDu:2018}). Thus in this paper, by taking advantage of the appealing flexibility and adaptability of the kernelized method we propose a new method which we call extrinsic kernel ridge regression classifier. Though, our main contribution of this paper is establishing a positive definite kernel on Kendall's planar shape space, similar strategies can be applied in a straightforward manner to other manifold valued data.

In order to establish a link between kernelized methods and shape analysis, let us first briefly recall the kernelized methods. Apart from a great surge of interest in analyzing data on manifolds, in the other areas of Statistics and Machine Learning, kernel methods have been successfully incorporated into a number of learning methods for which the algorithm depends only on the inner product of the data. Well-known examples of the kernelized methods include kernel ridge regression (KRRC), kernel Fisher discriminant analysis \citep[KFDA;][]{MiRaWeScMu:1999}, kernel principal component analysis \citep[KPCA;][]{ScSm1997}, and support vector machine (SVM) (For more examples, see \citet{ScSm:2002, HoScSm:2008}). Two main benefits of kernel methods are summarized as follows. (i) The effectiveness of extending linear methods to nonlinear methods. More precisely speaking, keeping a given learning algorithm unchanged, the kernel method enables converting a linear method to its nonlinear counterpart. Toward this end, we first assume that data from the original space is transformed into a high dimensional feature space via a nonlinear feature map. Then on that space the same algorithm can be carried out by substituting the inner product in the original space with that of the feature space. Due to what is called the ``kernel trick'', i.e., the inner products of the high dimensional feature space can be effectively computed in the original input space via the kernel function, a nonlinearity is obtained at no additional computational cost. (ii) The availability of applying the Euclidean algorithms to non-Euclidean spaces over which positive definite kernels can be defined. Informally speaking a kernel function is a similarity measure between two objects $\bx_i$ and $\bx_j$ on an input space $\mathcal{X}$. And because no assumption is made about the underlying structures of $\mathcal{X}$, kernels can be defined on arbitrary spaces. Thus, with this in consideration, the kernel method is not limited to an Euclidean vector space, but can be applied to various types of data domains such as sequence, networks, graphs, text, images, as well as manifolds. The methods can be theoretically justified by the theory of Reproducing Kernel Hilbert Space (RKHS), provided that the kernel function is positive definite. In light of those benefits, the kernel method should be considered as the first attempt to provide an efficient way to dealing with classification problems over non-Euclidean space.

The rest of this paper is organized as follows. In \cref{sec:2}, some background information and preliminaries are presented. In \cref{sec:KRRC}, we describe the kernel ridge regression classifier (KRRC) on the planar shape space $\Sigma_2^k$. In \cref{sec:extrinsic}, we develop the positive definite kernel on $\Sigma_2^k$, which we call extrinsic Veronese Whitney Gaussian kernel. In \cref{sec:simulation}, we illustrate our proposed methods on a real dataset. The paper is then concluded in \cref{sec:conclusion} with a discussion on future research directions.

\section{Preliminaries} \label{sec:2}

In this section, we present some necessary definitions and preliminary concepts that will be used throughout this paper. Let us start by recalling regression classifiers based on the subspace learning methods.

\subsection{Regression Classifiers}
Subspace learning methods for object classification \citep{BeHeKr1997, BaJa2003} have been extensively studied over the last few decades and successfully applied in many different settings within the field of face recognition. One pioneering and influential subspace classification method is linear regression classifier (LRC) developed by \citet{NaTo:2010}. The fundamental idea behind LRC is that a random object which belongs to a specific class is assumed to lie on a linear subspace spanned by observations in that class. More specifically, suppose that we have $C$ classes and each class has $n_i$ subsamples in the $p$-dimensional space. Then, according to the subspace assumption, the new data $\bx \in \mathbb{R}^p$ which belongs to the $i$ th class can be expressed in terms of a linear combination of the training sets from the $i$ th class; $\bx = \bX_{(i)}\bbeta_{(i)} + \beps$, where $\bX_{(i)} = [\bx_{(i)}^1, \cdots, \bx_{(i)}^{n_i}] \in \mathbb{R}^{p\times n_i}, \bbeta_{(i)}\in\mathbb{R}^{n_i}, \beps \in  \mathbb{R}^p$ denote the class-specific data matrix of the $i$th class, coefficient, and random error, respectively. The class specific regression coefficient can be obtained through the ordinary least square method $\hat{\bbeta}_{(i)} = \argmin_{\bbeta_{(i)}} \Vert \bx-\bX_{(i)}\bbeta_{(i)} \Vert_2^2$. After projecting a new sample data $\bx$ onto the subspaces of each different classes, the label of the new data can be estimated by minimizing distances between the given data point and the projected points, $\hat{\by} = \argmin_{i=\{1, \cdots, C\}} \Vert \bx - \hat{\bx}_{(i)} \Vert_2^2,$
where $\hat{\bx}_{(i)} = \bX_{(i)}\hat{\bbeta}_{(i)}$ denotes the projected point to the $i$th subspace. A schematic representation of the method is presented in \autoref{fig:LRC}.  After the remarkable success of LRC, to date, a variety of LRC based methods have been developed and achieved improved performance \citep{NaRoBe2012, HuYa2012}. Although such methods provide insight on how to apply a linear regression technique to a classification problem, LRC based methods has some drawbacks when attempting to perform shape classification on Kendall's planar shape space. Firstly, due to the linearity assumption intrinsically imposed on the model, LRC lacks the capability of capturing a nonlinear structure of samples. Even this linearity assumption would not be appropriate for shape classification, since the shape space has a nonlinear manifold structure. What is more, the ordinary least square (OLS) employed in estimating $\bbeta_{(i)}$ requires LRC to have $p \geq n_i$ in order that the system of equation $\bx = \bX_{(i)}\bbeta_{(i)}$ is well conditioned. To resolve these drawbacks, \citet{HeDi:2014} proposed kernel ridge regression classifier (KRRC) by incorporating the kernel method and ridge regression into LRC. Motivated by the success of KRRC on Euclidean space, in this paper, we propose a new shape classifier on $\Sigma_2^k$, which we refer to as extrinsic KRRC.

\begin{figure}
\begin{center}
\begin{tikzpicture}
%\filldraw[fill=gray!20,draw=gray!70,opacity=0.6] (0,0) -- (2,0) -- (2,2) -- (0,2) -- cycle;
\draw (1,-0.5) node {1 st class} ;
%\draw (1,0.8) node {$\hat{x}_{(1)}$} ;
\draw (2.5,1) node {$\cdots$} ;
\filldraw[fill=gray!20,draw=gray!70,opacity=0.6] (0,1) -- (2,0) -- (2,2) -- (0,3) -- cycle;
\draw (1,1.4) node {$\hat{\bx}_{(1)}$} ;

\filldraw[fill=gray!20,draw=gray!70,opacity=0.6]  (3,0) -- (5,0) -- (5,2) -- (3,2) -- cycle;
\draw (4,-0.5) node {$i$ th class} ;
\draw (4,0.8) node {$\hat{\bx}_{(i)}$} ;

\draw (5.5,1) node {$\cdots$} ;

%\filldraw[fill=gray!20,draw=gray!70,opacity=0.6](6,0) -- (8,0) -- (8,2) -- (6,2) -- cycle;
\filldraw[fill=gray!20,draw=gray!70,opacity=0.6](6,0) -- (8,1) -- (8,3) -- (6,2) -- cycle;
\draw (7,-0.5) node {$C$ th class} ;
%\draw (7,0.8) node {$\hat{x}_{(C)}$} ;
\draw (7,1.4) node {$\hat{\bx}_{(C)}$} ;

\draw (4,4) node {$\bx$};

%\draw [->,>=stealth] (3.7,3.7) -- (1,1);
\draw [solid] (3.7,3.7) -- (1.8,2.222222) ;
\draw [->,>=stealth,dashed] (1.8,2.222222) -- (1,1.6);
%\draw [->,>=stealth] (3.7,3.7) -- (1,1.6);
\draw [solid] (4,3.7) -- (4,2.15);
\draw [->,>=stealth,dashed] (4,2.15) -- (4,1.2);
%\draw [->,>=stealth] (4.3,3.7) -- (7,1);
\draw [->,>=stealth,dashed] (6.2,2.222222) -- (7,1.6);
\draw [solid] (4.3,3.7) -- (6.2,2.222222);
\end{tikzpicture}
\end{center}
\caption{Graphical interpretation of regression classifiers : $\bx$ is the new data point, and gray shaded regions display the subspaces of the corresponding classes, $\hat{\bx}_{(i)}$} is the projected value onto the $i$th subspace.\label{fig:LRC}
\end{figure}
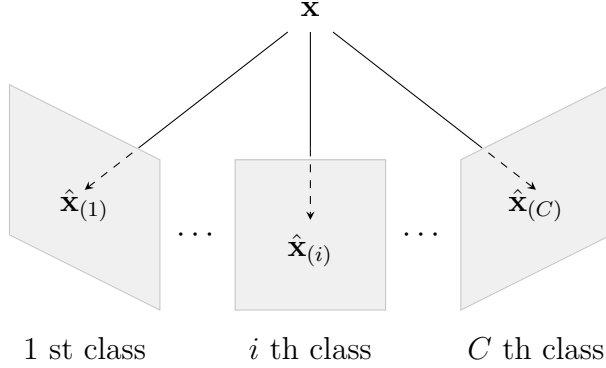

\subsection{Planar Shape Space and Extrinsic Analysis}\label{sec:planarshape}
The aim of this subsection is to provide the background information on Kendall's planar shape space of $k$-ads \citep{Kendall:1984}, also known as similarity shape space. $k$-ads is a landmark configurations that describe a shape of object. Now we begin by giving a brief introduction of the geometry of Kendall's planar shape space.

A similarity shape is defined as a geometrical object that is invariant under the Euclidean similarity transformations (i.e., translation, scale and rotational effects). Note that a general similarity shape space is denoted by $\Sigma_m^k$, where $k, m$ denote the number of landmarks, and the Euclidean dimension on which landmarks lie, respectively. It is also worth noting that unlike Kendall's planar shape space ($m=2$), which is a compact differentiable manifold, a similarity shape space $\Sigma_m^k$ with $m>2$ fails to have a manifold structure \citep{BhBh:2012}. The Kendall's planar shape space $\Sigma_2^k$ is identified in the following manner. First, landmarks are placed on a complex plane as a set of $k$ complex numbers $\bz = (z_1, \cdots, z_k)$, where $z_j = x_j + i y_j$.
Then the effect of translation and scaling can be filtered out by centralizing the configuration and dividing it by its size in order to have a unit norm
\begin{align*}
 \bu = \frac {\bz- \langle \bz \rangle}{\Vert \bz-\langle \bz\rangle \Vert} \in \mathbb{C}S^{k-1} ,
\end{align*}
where $\langle \bz \rangle=(\frac{1}{k}\sum_{i=1}^k z_i, \cdots, \frac{1}{k}\sum_{i=1}^k z_i)$, and $\bu$ is referred to as the preshape of $\bz$. The above procedure maps the given configuration to the point in the pre-shape space, which is equivalently regarded as the space of complex hypersphere $\mathbb{C}S^{k-1} = \left\{ \bu \in \mathbb{C}^k \vert \sum_{i=1}^k u_j = 0, \Vert \bu \Vert = 1 \right\}$. Since, the shape is obtained by eliminating rotational effect from the preshape, the shape of $\bz$ which is denoted as $[\bz]$, is eventually identified as a collection of rotated versions of $\bu$ by angle $\theta$ ; $[\bz] = \left\{ e^{i\theta}\bu : 0 \leq \theta < 2\pi \right\}$.
%\begin{align*}
%[\bz] = \left\{ e^{i\theta}\bu : 0 \leq \theta < 2\pi \right\}.
%\end{align*}
Although this fully generalizes the similarity shape on a plane, it is useful to note that another important characterization of $\Sigma_2^k$ can be obtained by rewriting the above in the following compact form : $\{ \lambda (\bz-\langle \bz \rangle) : \lambda \in \mathbb{C}\setminus\{0\} \}$, where $\lambda = r e^{i\theta}, r > 0$ and $0 \leq \theta <2\pi$. In this expression, we simultaneously consider the scaling by a scalar $r$ and rotation by an angle $\theta$ via multiplying the centralized $k$-ad configuration $\bz-\langle \bz \rangle$ by the complex number $\lambda$. Thus, the planar shape space is identified as the complex projective space $\mathbb{C}P^{k-2}$, the space of all complex lines through the origin in $\mathbb{C}^{k-1}$.  
In order to provide a more clear insight into the similarity shape, a graphical illustration of the two different shapes is given in \Cref{fig:similarity}. Parts of the data used in this illustration were collected by \citet{StGo2002}. For a more detailed explanation about Kendall's planar shape space, see \citet{DrMa:2016}, and \citet{PaEl:2015}.
\begin{figure}
\centering
    \includegraphics[width=0.8\linewidth]{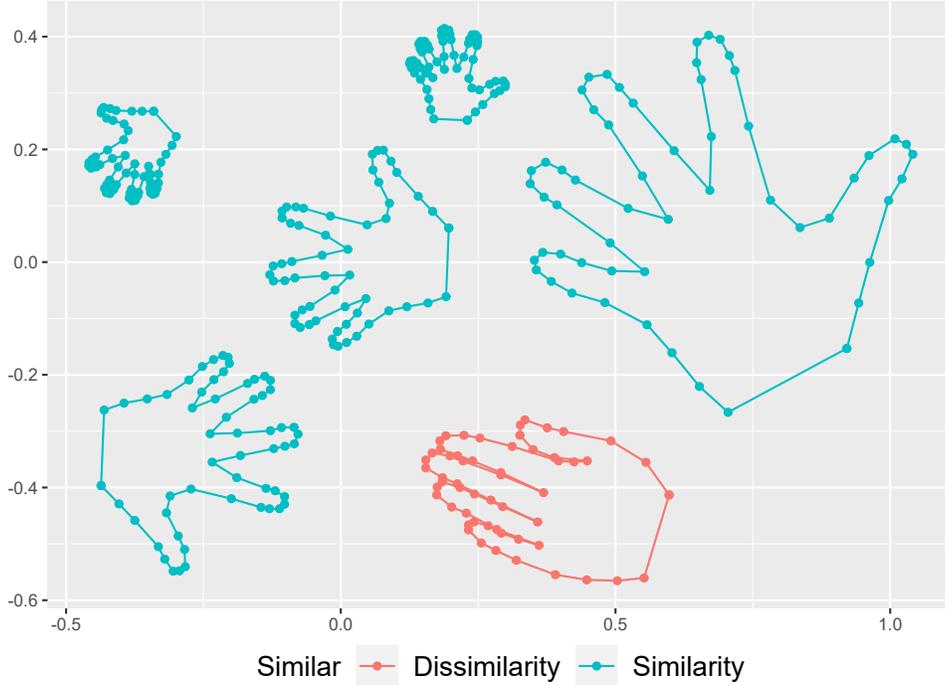}
    \caption{The five blue hands have a same shape under different similarity transformations. That means they would be located at exactly the same point on $\Sigma_2^k$. However, the shape of the red hand with contracted fingers does not match any of the blue hands. It indicates that there exists no similarity transformation such that maps the shape of the red hand into the shapes of blue hands.}
    \label{fig:similarity}
\end{figure}

To better understand our proposed kernel on $\Sigma_2^k$, the rest of this section describes the necessary background on manifold valued data analysis. In the literature of manifold-valued data analysis, two types of approaches caused by using different types of distances have been considered. First, a natural choice of distances is the Riemannian metric which leads to the intrinsic data analysis on manifolds. Despite their popularity in manifold valued data analysis, however, the intrinsic Riemannian distance has not been widely utilized in kernel methods due to difficulty in constructing a positive definite kernel. For example the Gaussian kernel equipped with the known intrinsic Riemannian distance between two shape $[\bz_i], [\bz_j] \in \Sigma_2^k$ 
\begin{align}
    \kappa([\bz_i],[\bz_j]) := \exp \left( -\frac{\rho_{R}^2([\bz_i],[
\bz_j])}{\sigma^2} \right),\label{eq:Int_Gaussian}
\end{align} where $\rho_{R} = \arccos(\vert \langle \bu_i, \bu_j \rangle \vert)$ is the Riemannian distance defined on $\Sigma_2^k$ and $\bu_i, \bu_j$ are pre-shapes of $[\bz_i], [\bz_j]$, respectively, doesn't lead to a positive definite kernel. Motivated by this problem, we instead consider a kernel with another distance on $\mathcal{M}$, called the extrinsic distance. The extrinsic distance which leads to the extrinsic data analysis on $\mathcal{M}$, is a chord distance in a higher-dimensional Euclidean space induced by an \emph{embedding} $J : \mathcal{M} \rightarrow \mathbb{R}^D$. The main benefit of using the extrinsic distance lies in the ability to construct positive definiteness kernels on general manifolds without suffering from the non-positive definiteness problem aroused by intrinsic distance. This desirable property will be demonstrated in \cref{sec:extrinsic}. We now finish this section by reviewing the extrinsic data analysis on manifolds.  

 The basic idea of extrinsic analysis is supported by the fact that a manifold can be embedded in a higher-dimensional Euclidean space (\citet{Wh:1944}) via an embedding $J$. Let us recall the following definition of embedding $J$ of a manifold $\mathcal{M}$ into an Euclidean space $\mathbb{R}^D$. In differential topology, an embedding is defined as a particular type of immersion. An immersion of a manifold $\mathcal{M}$ into an Euclidan space $\mathbb{R}^D$ is a differentiable map $J : \mathcal{M} \rightarrow \mathbb{R}^D$, for which their differential $d_pJ : T_p\mathcal{M} \rightarrow T_{J(p)}\mathbb{R}^D$ is a one-to-one, where $T_p\mathcal{M}$,  $T_{J(p)}\mathbb{R}^D$ denote the tangent space of $\mathcal{M}$ at $p \in \mathcal{M}$ and the tangent space of $\mathbb{R}^D$ at $J(p)$, respectively. Then a one-to-one immersion $J : \mathcal{M} \rightarrow \mathbb{R}^D$ is called embedding if $J$ is a homeomorphism from $\mathcal{M}$ to $J(\mathcal{M})$ with the induced topology. It should be noted that unfortunately the embedding $J$ is not uniquely determined. Regarding the non-uniqueness issue of embedding, the question that immediately arises is which embedding should be used ? Among many possible choices of embeddings, in the literature of the extrinsic analysis, the equivariant embedding is generally preferred due to its capability of conserving a geometry of the given manifold. The following is the definition of the equivariant embedding. For a Lie group $G$ acting on $\mathcal{M}$, the embedding $J : \mathcal{M} \rightarrow \mathbb{R}^D $ is said to be the $G$ equivariant embedding if there exists the group homomorphism
$\phi : G \rightarrow \operatorname{GL}_D(\mathbb{R} \ \text{or} \ \mathbb{C})$ satisfying $J(gp) = \phi(g) J(p), \forall p \in \mathcal{M}, g \in G$, where $\operatorname{GL}_D(\mathbb{R} \ \text{or} \ \mathbb{C})$ denotes the general linear group which is the group of $D\times D$ invertible real, or complex matrices, respectively. This definition indicates that the group action on the embedded space $J(\mathcal{M})$ is recovered by the corresponding group action of $G$ on $\mathcal{M}$ through $\phi$. So in this sense, the important geometrical property of the manifold is preserved by the equivariant embedding. For the planar shape space, the Veronese Whitney (VW) embedding \citep{Ke:1992, BhBh:2012},
\begin{align}
	J : & \Sigma_2^k \rightarrow S(k,\mathbb{C}) , \ \text{such that} \ [\bz] \mapsto \bu \bu^\ast ,
	\label{eq:VW_embedding}
\end{align}
which maps from $\Sigma_2^k$ to $S(k,\mathbb{C})$, the space of $k \times k$ self-adjoint (or Hermitian) matrices, is typically used. It can be easily shown that the Veronese Whitney embedding is the $\operatorname{SU}(k)$ equivariant embedding, where $\operatorname{SU}(k) = \left\{ A \in \operatorname{GL}(k,\mathbb{C}), AA^\ast = I, \det(A) = 1 \right\}$ is the special unitary group that acts on $\Sigma^k_2$. Since $A [\bz] = [A\bz]$ satisfies for $A \in \operatorname{SU}(k), [\bz] \in \Sigma_2^k$, we have
\begin{align*}
	J (A [\bz]) &= A \bu \bu^\ast A^\ast\\
	&= \phi(A) \bu \bu^\ast.
\end{align*}
The last line above is obtained by taking $\phi : \operatorname{SU}(k) \rightarrow \operatorname{GL}_k(\mathbb{C})$ such that $\phi(A)B = ABA^\ast$, where $B \in S(k,\mathbb{C})$. Thus by using the definition of Veronese Whitney embedding \eqref{eq:VW_embedding}, we have shown that $J(A[\bz]) = \phi(A) J([\bz])$, which completes the proof that the Veronese Whitney embedding is the $\operatorname{SU}(k)$ equivariant.

%%%%%%%%%%%%%%%%%%%%%%%%%%%%%%%%%%%%%%%%%%%%%%%%%%%%%%%%%%%%%%%%%%%%%%%%%%%%%%
%% FPG
%%%%%%%%%%%%%%%%%%%%%%%%%%%%%%%%%%%%%%%%%%%%%%%%%%%%%%%%%%%%%%%%%%%%%%%%%%%%%%
\section{Kernel Ridge Regression Classifier on $\Sigma_2^k$}\label{sec:KRRC}

\subsection{Naive RRC}\label{sec:NaiveRRC}

Before introducing our method, we here describe a naive way of applying existing regression classifier to the shape space. Suppose we observe $\{ [\bz_i], \by_i \}_{i=1}^n $, where $[\bz_i] \in \Sigma_2^k$, and $\by_i \in \mathcal{Y} = \{ 1,\cdots, C\}$. If shapes are uncautiously treated as $k$ dimensional complex random vectors, then preshape data matrix of the $i$ th class $\bU_{(i)} = \left[\bu_{(i)}^1, \bu_{(i)}^2,  \cdots , \bu_{(i)}^{n_i} \right] \in \mathbb{C}^{k \times n_i}$ is obtained by stacking each preshape column-wise. Then under the linear subspace assumption, the naive ridge regression classifier for the shape manifold can be formulated based on the following class specific regression model
\begin{align}
	\bu = \bU_{(i)} \bbeta_{(i)} + \beps \label{eq:linear_subspace_reg},
\end{align}
where $\bbeta_{(i)} \in \mathbb{C}^{n_i}$, $\bu$ and $\beps$ are the regression coefficient, the new preshape, and a $\mathbb{C}^k$ valued error, respectively. Here instead of using the usual OLS estimation that minimizes the sum of squared residuals, we adopt the ridge regression approach which minimizes the following complex valued regularized least squares problem
\begin{align}
	\hat{\bbeta}_{(i)} &= \argmin_{\bbeta_{(i)}} \Vert \bu - \bU_{(i)}\bbeta_{(i)} \Vert_2^2  + \lambda \Vert \bbeta_{(i)} \Vert_2^2, \label{eq:complex_reg}
\end{align}
where $\lambda > 0$ is a regularization parameter. The least squares problem for real-valued variables can be straightforwardly generalized to complex-valued variables by replacing matrix transposes with conjugate transposes. Thus the ridge regression coefficient $\hat{\bbeta}_{(i)} = \left(\bU_{(i)}^\ast \bU_{(i)} + \lambda I\right)^{-1} \bU_{(i)}^\ast \bu$, where $\bU_{(i)}^\ast$ is the conjugate transpose of $\bU_{(i)}$, and the projected value onto the $i$ th subspace $\hat{\bu}_{(i)} = \bU_{(i)} \hat{\bbeta}_{(i)} = \bU_{(i)} \left(\bU_{(i)}^\ast \bU_{(i)} + \lambda I\right)^{-1} \bU_{(i)}^\ast \bu$ are obtained. One advantage to be gained by adopting ridge estimator is that it does not necessarily require $k > n_i$ which is a common assumption imposed in LRC. Moreover, it resolves the multicollinearity issue among training samples. The label of the given shape is finally predict by minimizing the distance between the projected complex vector onto each subspace and $\bu$ 
\begin{align}
	\hat{\by}_\bu &= \argmin_{i=\{1,\cdots, C\}} \Vert \hat{\bu}_{(i)} - \bu \Vert_2^2 \notag\\
	&= \argmin_{i} \Vert \bU_{(i)} \left(\bU_{(i)}^\ast \bU_{(i)} + \lambda I\right)^{-1} \bU_{(i)}^\ast \bu - \bu\Vert_2^2 \label{eq:naive_classification} .
\end{align}
While on the surface it seems a reasonable way forward, in fact, the approach described above has the following conceptual limitations. First, note that because the planar shape space is a nonlinear manifold, the linear subspace assumption in \eqref{eq:linear_subspace_reg} is not appropriate. Second, the geometry of the shape space is not taken into account by carelessly employing Euclidean norm for a preshape in the first part of \eqref{eq:complex_reg}, and \eqref{eq:naive_classification}. As will be illustrated in \cref{sec:simulation}, that might be a main cause yielding a poor estimation performance of the naive RRC. This indicates that due to the nonlinear manifold structure of the planar Kendall shape space, it is not straightforward to apply RRC to shape classification.

\subsection{KRRC on $\Sigma_2^k$}
The aforementioned problems encountered when directly applying the RRC to $\Sigma_2^k$ is mostly caused by the lack of the nonlinearity in the model. This can be efficiently resolved by exploiting the RKHS method. Suppose a nonlinear function $\Phi : \mathcal{X} \rightarrow \mathcal{H}$ that maps the data values lying on the original input space into a high dimensional feature space which, in fact, is a reproducing kernel Hilbert space (RKHS) of functions. Under the linear subspace assumption in a feature space $\mathcal{H}$, the class specific ridge regression \eqref{eq:complex_reg} can be extended to $\mathcal{H}$ in the following manner
\begin{align}
	\hat{\bbeta}_{\Phi, (i)} = \argmin_{\bbeta_{(i)}} \Vert \Phi(\bu)-\PPhi_{(i)} \bbeta_{(i)} \Vert_2^2 + \lambda \Vert \bbeta_{(i)} \Vert_2^2 \label{eq:ker_ridge}.
\end{align}
We denote by $\PPhi_{(i)}$ an $(\text{dim}(\mathcal{H}) \times n_i)$ matrix of regressors whose $j$ th column is the feature vector $\Phi(\bu_{(i)}^j)$ of the $j$ th element in the $i$ th class.
The solution of \eqref{eq:ker_ridge} is then given by, $\hat{\bbeta}_{\Phi, (i)} = \left(\PPhi_{(i)}^\top \PPhi_{(i)}+ \lambda I \right)^{-1} \PPhi_{(i)}^\top \Phi(\bu)$. From the RKHS theory, the feature space $\mathcal{H}$ can be implicitly determined by the kernel function $\kappa(\cdot,\cdot) : \mathcal{X} \times \mathcal{X} \rightarrow \mathbb{R}$. Precisely, the kernel function $\kappa$ plays a role of the inner product in the feature space by satisfying $\kappa(\bx,\by) = \langle \Phi(\bx), \Phi(\by) \rangle$, where $\bx,\by \in \mathcal{X}$. Then we see that $\PPhi_{(i)}^\top \PPhi_{(i)}$ can be represented as the $(n_i \times n_i)$ Gram matrix $\bK_{(i)}$ whose $(l, j)$th element $\bK_{(i)}^{(l,j)}$ is $ \langle\Phi(\bu_{(i)}^l), \Phi(\bu_{(i)}^j)\rangle =\kappa(\bu_{(i)}^l,\bu_{(i)}^j)$. Thus the solution of \eqref{eq:ker_ridge} can be rewritten as $\hat{\bbeta}_{\Phi, (i)} = (\bK_{(i)} + \lambda I)^{-1} \bk_{(i)}$, where $\bk_{(i)} = (\kappa(\bu_{(i)}^1, \bu), \cdots, \kappa(\bu_{(i)}^{n_i}, \bu))^\top$ is the vector of inner products in the feature space between the given training data of the $i$ th group and the new point $\bu$. The projected value of $\bu$ onto the feature subspace of the $i$ th class is given by 
\begin{align}
	\hat{\Phi}_{(i)}(\bu) &= \PPhi_{(i)}\hat{\bbeta}_{\Phi, (i)} = \PPhi_{(i)} (\bK_{(i)} + \lambda I)^{-1} \bk_{(i)} \label{eq:Proj_KRRC} .
\end{align}
As can be observed from the above equation, the projection depends on the feature matrix. Since the dimension of the feature space is usually very high or even infinite, it is difficult to directly evaluate $\hat{\Phi}_{(i)}(\bu)$. By using some algebra, however, it can be shown that the final prediction phase given in the following
\allowdisplaybreaks[1]
\begin{align}
	\hat{\by}_\bu &= \argmin_{i=\{1,\cdots,C\}} \Vert \hat{\Phi}_{(i)}(\bu) - \Phi(\bu) \Vert_2^2 \notag\\
	&= \argmin_{i=\{1,\cdots,C\}} \hat{\Phi}_{(i)}(\bu)^\top \hat{\Phi}_{(i)}(\bu) - 2 \hat{\Phi}_{(i)}(\bu)^\top \Phi(\bu) \notag\\
	&=  \argmin_{i=\{1,\cdots,C\}} \bk_{(i)}(\bu)^\top \left(\bK_{(i)} + \lambda I\right)^{-1} \bK_{(i)} \left(\bK_{(i)}+\lambda I\right)^{-1}\bk_{(i)}(\bu)	-2 \bk_{(i)}(\bu)^\top \left(\bK_{(i)} + \lambda I \right)^{-1} \bk_{(i)}(\bu) \notag\\
	&=  \argmin_{i=\{1,\cdots,C\}} \bk_{(i)}(\bu)^\top  \left(\bK_{(i)} + \lambda I\right)^{-1} \left(-\bK_{(i)} - 2 \lambda I \right) \left(\bK_{(i)} + \lambda I\right)^{-1} \bk_{(i)}(\bu),
	\label{eq:kernel_trick}
\end{align}
is no longer depending on the feature map $\Phi(\cdot)$.
Thus, without evaluating the unknown feature map, the kernel trick enables KRRC to efficiently classify shapes via kernel function defined on the input space.

%%%%%%%%%%%%%%%%%%%%%%%%%%%%%%%%%%%%%%%%%%%%%%%%%%%%%%%%%%%%%%%%%%%%%%%%%%%%%%
%% Extrinsic KRR
%%%%%%%%%%%%%%%%%%%%%%%%%%%%%%%%%%%%%%%%%%%%%%%%%%%%%%%%%%%%%%%%%%%%%%%%%%%%%%

\section{Extrinsic Kernel on $\Sigma_2^k$}\label{sec:extrinsic}

Regarding the selection of the kernel, the most popular Euclidean kernel used in various kernelized methods is the Gaussian radial basis function (RBF) kernel, which maps given data into the infinite dimensional feature space. Recall the Euclidean Gaussian RBF kernel for two given points $\bx_i, \bx_j \in \mathbb{R}^d$ 
\begin{align}
	\kappa(\bx_i,\bx_j) := \exp \left( - \frac{\Vert \bx_i - \bx_j \Vert^2}{\sigma^2}\right). \label{eq:Euclidean_GRBF}
\end{align}
To perform kernelized methods on $\Sigma_2^k$, one might consider the Gaussian RBF kernel by substituting the Euclidean distance in \eqref{eq:Euclidean_GRBF} with any particular shape distance chosen by a researcher's personal preference and specific needs. Then the Gaussian like kernel on the planar shape space takes the form of $\kappa([\bz_i],[\bz_j]) := \exp \left( -d_{\Sigma_2^k}^2([\bz_i],[
\bz_j])/\sigma^2 \right)$, where $d_{\Sigma_2^k}(\cdot,\cdot)$ denotes a generic distance function on $\Sigma_2^k$. Some examples of $d_{\Sigma_2^k}$ include the arc length (Riemannian distance), the partial Procrustes distance, and the full Procrustes distance. The kernel obtained as described above might seem to be tempting to generalize the Gaussian RBF kernel on the shape space, as no additional effort is required to conduct kernelized methods on that space. But unfortunately, not all choices of distances lead to a positive definite kernel, which is an essential requirement of RKHS methods to construct the valid feature space. For example, as we have seen in \Cref{sec:planarshape} the Gaussian kernel equipped with intrinsic Riemannian distance \eqref{eq:Int_Gaussian} fails to be a positive definite kernel on $\Sigma_2^k$. To the best our knowledge, the positive definite Gaussian like kernel on $\Sigma_2^k$ was firstly proposed by \cite{JaSaLiHa:2013} using the full Procrustes distance $d_{\text{FP}}([\bz_i],[\bz_j]) = \left(1- \vert \langle \bu_i, \bu_j \rangle \vert^2\right)^{1/2}$. The kernel referred to as the full Procrustes Gaussian (FPG) kernel is given by
\begin{align}
	\kappa([\bz_i],[\bz_j]) := \exp \left( - \frac{ 1- \vert \langle \bu_i, \bu_j \rangle \vert^2 }{\sigma^2}\right), \label{eq:full_GRBF_Ker}
\end{align}
where $\bu_i, \bu_j \in \mathbb{C}S^{k-1}$ are given pre-shape of $[\bz_i]$, and $[\bz_j]$, respectively. Various kernelized methods for the planar Kendall shape space have been successfully implemented using the FPG kernel. Thus the FPG kernel could be considered as a potential candidate kernel that can be directly exploited for KRRC on $\Sigma_2^k$.

As we have noted before, the kernel methods have been applied to various types of data, since no assumption is required on a domain of a data space, other than it being nonempty set. But the above metric based approach for derivation of positive definite kernel is not capable of providing a general way for utilizing kernel methods on non Euclidean space. Suppose we perform kernelized methods on an arbitrary manifold. In such situation, the main drawback of the above substituting distance method is that for each space of interest it requires to find a specific distance that ensures a positive definiteness kernel. However, it is usually difficult to check the Gaussian kernel equipped with a chosen distance leads to a positive definite kernel. For example, \citet{FeHa:2016} indicated that positive definiteness of the Gaussian kernel is violated unless the input metric space $(\mathcal{X},d)$ is flat in the sense of Alexandrov \citep{BrHa:1999}. Thus in practice, this approach can not be directly adopted for implementing the kernel methods on arbitrary manifolds. So in what follows, instead of directly using distances defined on the Kendall's planar shape manifold, we present a facile and universal approach that always guarantees positive definiteness on manifolds.
%%%%%%%%%%%%%%%%%%%%%%%%%%%%%%%%%%%%%%%%%%%%%%%%%%%%%%%%%%%%%%%%%%%%%%%%%%%%%%
%% Extrinsic VW Gaussian Kernel
%%%%%%%%%%%%%%%%%%%%%%%%%%%%%%%%%%%%%%%%%%%%%%%%%%%%%%%%%%%%%%%%%%%%%%%%%%%%%%

\subsection{Extrinsic Veronese Whitney Gaussian kernel}
In this section, we propose the Gaussian RBF kernel on $\Sigma_2^k$ which makes use of the induced Euclidean distance between two shapes $[\bz_i]$, and $[\bz_j]$ via VW embedding \eqref{eq:VW_embedding}. The proposed kernel which we call extrinsic Veronese Whitney Gaussian (VWG) kernel is given by,  
\begin{align}
	\kappa([\bz_i],[\bz_j]) = \exp \left( -\frac{\rho_E^2 \left([\bz_i],[\bz_j]\right) }{\sigma^2} \right), 
	\label{eq:VW_Gaussian}
\end{align}
where $\rho_E^2 ([\bz_i],[\bz_j]) = \tr \left[(J([\bz_i])-J([\bz_j]))^\ast(J([\bz_i])-J([\bz_j]))\right]$. Since the extrinsic distance can be rewritten as $\sum_{l,l^\prime} \Vert (\bu_i\bu_i^\ast)_{l,l^\prime} - (\bu_j \bu_j^\ast)_{l,l^\prime} \Vert^2$, it is naturally the squared Euclidean distance between two $k \times k$ complex Hermitian matrices $J([\bz_i])$ and $J([\bz_j])$ regarded as elements of $\mathbb{C}^{k^2}$. Moreover the squared extrinsic distance via VW embedding is associated with the full Procrustes distance in the following manner;
\begin{align*}
\rho_E^2 ([\bz_i],[\bz_j]) &= \operatorname{Trace}(\bu_i \bu_i^\ast - \bu_j\bu_j^\ast)^2\\
&= \sum_{l=1}^k \vert \bu_i^l \vert^2 + \sum_{l=1}^k \vert \bu_j^l \vert^2 - \sum_{l=1}^k\sum_{l^\prime=1}^k \left(\bu_i^l \bar{\bu}_i^{l^\prime}\bu_j^{l^\prime} \bar{\bu}_j^l +  \bu_j^l \bar{\bu}_j^{l^\prime}\bu_i^{l^\prime} \bar{\bu}_i^l\right)\\
&= 2 - 2 \vert \bu_i^\ast \bu_j \vert^2\\
&= 2d_{FP}^2 ([\bz_i],[\bz_j]),
\end{align*}
where the superscript $l$ denotes the $l$ th element of the preshape, and the bar above the preshape denotes the complex conjugate. And one can show both the full and partial Procrustes distances ($ d_{P}^2 ([\bz_i],[\bz_j]) = (1-\vert \langle \bu_i,\bu_j \rangle \vert) $) are extrinsic type distances because they are distances based on some embedded spaces. 

Though the full Procrustes Gaussian kernel and the extrinsic Veronese Whitney Gaussian kernel corresponds to each other in the planar shape space, it is worth emphasizing that the our extrinsic approach takes one promising benefit to construct kernel on general manifolds. In terms of constructing a kernel on manifold, the advantage to be obtained by exploiting extrinsic framework rather than substituting distance method is that it resolves non-positive definiteness issue which is potentially inherent in the intrinsic Riemannian distance. More precisely, once a proper embedding $J : \mathcal{M} \rightarrow \mathbb{R}^D$ exists, an extrinsic distance induced by $J$ guarantees the positive definiteness of Gaussian kernel. 

The rest of this section is devoted to prove that the Veronese Whitney Gaussian kernel is positive definite. To this end, we need the following technical result in \citet{BeChRe:1984}. A kernel having the form of $\exp(-t f(\bx_i, \bx_j))$ is positive definite for all $t > 0$, if and only if $f$ is negative definite function \citep[for proof, see pp.74-75 in][Theorem 2.2]{BeChRe:1984}. The above result mainly originated from \citet{Sc:1938}. We recall that for any nonempty set $\mathcal{X}$, a function $f :(\mathcal{X} \times \mathcal{X} ) \rightarrow \mathbb{R}$ is negative definite if and only if $f$ is symmetric, and $\sum_{i,j=1}^n \alpha_i \alpha_j f(\bx_i,\bx_j) \leq 0$, for all $n \in \mathbb{N}, \{ \bx_1, \cdots, \bx_n \} \subseteq \mathcal{X}$, and $\{\alpha_1,\cdots, \alpha_n\} \subseteq \mathbb{R}$ with $\sum_{i=1}^n \alpha_i = 0 $. Therefore, using the above fact, it suffices to verify that the squared extrinsic distance $\rho_E^2$ is negative definite, which is stated in the following theorem. 

\begin{thm} \label{thm:vw_neg}
	The squared extrinsic Euclidean distance function $\rho_E^2 : (\Sigma_2^k \times \Sigma_2^k) \rightarrow \mathbb{R} $ induced by the Veronese Whitney embedding $J : \Sigma_2^k \rightarrow S(k,\mathbb{C})$ 
	\begin{align*}
		\rho_E^2 ([\bz_i],[\bz_j]) &:= \Vert J[\bz_i]-J[\bz_j] \Vert_F^2\\
		&= \tr \displaystyle\left\{ (J[\bz_i]-J[\bz_j] )^\ast((J[\bz_i]-J[\bz_j])\right\}, 
	\end{align*}
where $[\bz_i],[\bz_j] \in \Sigma_2^k$, is negative definite.
\end{thm}
\begin{proof} The following proof is immediately from Theorem 4.4 in \citet{JaHaSaLiHa:2013},
\begin{align*}
	\sum_{i,j=1}^n \alpha_i \alpha_j \rho_E^2 (J([\bz_i]), J([\bz_j])) &=\sum_{i,j=1}^n \alpha_i \alpha_j \Vert J([\bz_i]) - J([\bz_j]) \Vert_F^2\\
	&=\sum_{i,j=1}^n \alpha_i \alpha_j \langle J([\bz_i]) - J([\bz_j]), J([\bz_i]) - J([\bz_j]) \rangle_F\\
	&=\sum_{j=1}^n \alpha_j \sum_{i=1}^n \alpha_i \langle J([\bz_i]), J([\bz_i]) \rangle_F -2 \sum_{i,j=1}^n \alpha_i \alpha_j \langle J([\bz_i]), J([\bz_j]) \rangle_F \\
	& \hspace{.5cm}+ \sum_{i=1}^n \alpha_i \sum_{j=1}^n \alpha_i \langle J([\bz_j]), J([\bz_j]) \rangle_F\\
	&= -2\sum_{i,j=1}^n\alpha_i \alpha_j \langle J([\bz_i]), J([\bz_j]) \rangle_F = -2 \Vert \sum_{i=1}^n \alpha_i J([\bz_i]) \Vert_F^2\\
	&\leq 0.
\end{align*}
\end{proof}
We have proved that the VW Gaussian kernel \eqref{eq:VW_Gaussian} is a positive definite kernel. Though \Cref{thm:vw_neg} only focused on the planar shape space, the result holds for general manifolds, because no specific assumption has been required about $\mathcal{M}$ and $J$, other than injective immersion which satisfying homeomorphism from $\mathcal{M}$ to $J(\mathcal{M})$. We further note that regarding the negative definiteness of the squared extrinsic distance, a metric space $(\mathcal{X}, d)$ is said to be a negative type if $\sum_{i,j\leq n} \alpha_i \alpha_j d(\bx_i,\bx_j) \leq 0$ holds for all $n \leq 1, \bx_1, \cdots, \bx_n \in \mathcal{X}$, and $\alpha_1,\cdots, \alpha_n \in \mathbb{R}$ with $\sum_{i=1}^n \alpha_i = 0 $ \citep{Ly:2013}. Moreover, as revealed by \citet{Sc:1937, Sc:1938}, a negative type metric space is equivalent to the embeddability into Hilbert space in the following manner. A metric space $(\mathcal{X},d)$ is of negative type if and only if there exist a Hilbert space $\mathcal{H}$ and a map $\phi : \mathcal{X} \rightarrow \mathcal{H}$ such that $\forall \bx, \bx^\prime \in \mathcal{X}, d(\bx,\bx^\prime) = \Vert \phi(\bx) - \phi(\bx^\prime) \Vert^2$. Therefore, a general manifold equipped with squared extrinsic distance $(\mathcal{M}, \rho_E^2)$ is shown to be a metric space of negative type, by taking $\phi = J$. This suggests that our extrinsic approach provides a very general framework for deriving positive definite Gaussian kernels on manifolds.

%%%%%%%%%%%%%%%%%%%%%%%%%%%%%%%%%%%%%%%%%%%%%%%%%%%%%%%%%%%%%%%%%%%%%%%%%%%%%%
%% Simulation
%%%%%%%%%%%%%%%%%%%%%%%%%%%%%%%%%%%%%%%%%%%%%%%%%%%%%%%%%%%%%%%%%%%%%%%%%%%%%%

\section{Real Data analysis}\label{sec:simulation}
In this section, we illustrate our proposed method by examining the {\tt PassifloraLeaves} data. The leaves of Passiflora, a botanical genus of more than 550 species of flowering plants, are remarkably different with respect to their species. {\tt PassifloraLeaves} data was collected by \citet{Ch:2016}. They analyzed shapes of  3,319 Passiflora leaves from 40 different species which were assigned into 7 classes according to their appearance. Fifteen landmarks were placed at homologous positions to capture the shapes leaves (see \Cref{fig:leaf_landmarks}). The data are available at \url{https://github.com/DanChitwood/PassifloraLeaves}, and a more detailed description can also be found in \citet{Ch:2017}. In \Cref{fig:leaf_classes}, shape differences of leaves among seven distinct groups are graphically demonstrated along with their sample extrinsic means which are obtained by minimizing the empirical Fr\'echet function. To be more specific, the extrinsic mean in a general manifold is given by 
\begin{align*}
    \hat{\mu}_E = J^{-1} \left( \mathcal{P}\Big(\argmin_{\px \in \mathbb{R}^D} \sum_{i=1}^n\Vert J(\bx_i) - \px \Vert^2\Big) \right), \ \bx_i \in \mathcal{M},
\end{align*}
where $\mathcal{P}$ is the projection onto the image of the embedding $J(\mathcal{M})$. In particular, the extrinsic mean for Kendall's planar shape space, also known as VW mean, is obtained by the unit eigenvector corresponding to largest eigenvalue of the complex matrix $\frac{1}{n}\sum_{i=1}^n \bu_i \bu_i^\ast$. For more detailed explanations of the extrinsic mean on $\Sigma_2^k$, see \citet{BhPa:2003}. 

\begin{figure}
	\centering
	\includegraphics[width=0.95\linewidth]{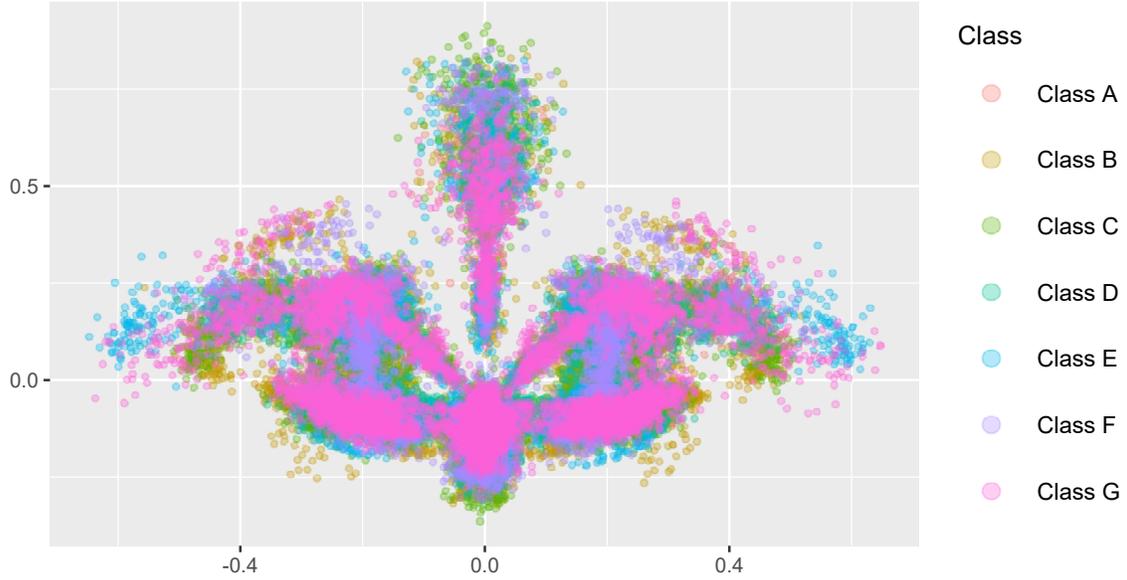} 
	\caption{The visualization of the {\tt PassifloraLeaves} data : 15 landmarks are placed on the leaves of Passiflora. The plot is colored by 7 different classes.}\label{fig:leaf_landmarks}
\end{figure}

\begin{figure}
	\centering
	\includegraphics[width=0.8\linewidth]{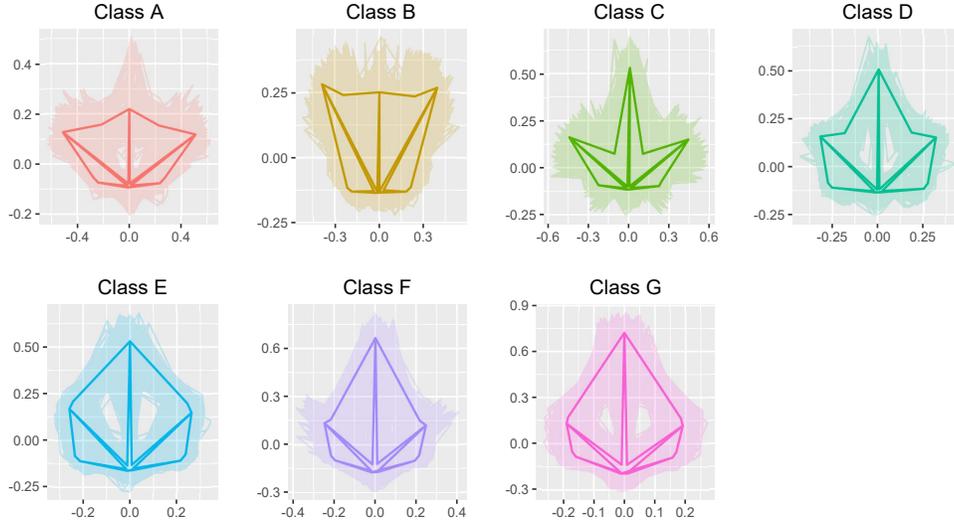}
	\caption{The shapes of leaves with different classes : The set of transparent lines represents observations for each of the classes, and the bold solid lines correpond to their extrinsic means.}\label{fig:leaf_classes}
\end{figure}
In addition to the KRRC with the extrinsic VW Gaussian kernel (VWG), we also consider KRRC with the intrinsic Riemannian Gaussian kernel in \eqref{eq:Int_Gaussian}, the Naive RRC described in \cref{sec:NaiveRRC}, Support Vector Machine (SVM) with Gaussian kernel implemented in the \texttt{\textbf{e1071}} package, Kernel Fisher discriminant analysis (KFA) implemented in the \texttt{\textbf{kfda}} package, and the multiclass GLM with the ridge penalty implemented in the {\fontfamily{lmtt}\selectfont glmnet} package as competing methods. And the software which implements the subspace learning based methods described in this paper, is available online at \url{https://github.com/hwiyoungstat/ShapeKRRC}. Note that performance of the aforementioned models depend significantly on a careful choice of tuning parameters. For examples, KRRC involves two tuning parameters that need to be determined in a data driven manner; (i) the regularization parameter $\lambda$ imposed on the ridge regression term, (ii) the kernel specific parameter $\sigma^2$ in the Gaussian kernel. However, since no relevant guideline concerning how to tune those parameters exists, for each run of the simulation the optimal combination of $(\lambda, \sigma^2)$ is jointly chosen by a two-dimensional grid search. At the beginning of every simulation, the samples in each of the 7 classes were randomly split into 60\% training and the ramaining 40\% test set. In addition, we drew subsamples of equal size from the training sets in order to construct subspaces. The number of subsamples for each class were varied with $n_i = \{10, 20, \cdots, 100\}$. Therefore, all model considered in this simulation were trained based only on a subset of the training set (size is $7 \times n_i$) and the performance was evaluated on the test set which was not used in the training phase. And simulations  were  replicated  20  times  in each of the subspace sizes. For evaluation metrics the macro averaging precision, recall, $F_1$ scores, and average accuracy are considered which are given by $ \text{prec} = \sum_{i=1}^{c} \frac{TP_i}{TP_i + FP_i}/c$, $\text{rec} = \sum_{i=1}^c \frac{TP_i}{TP_i + FN_i}/c$, $F_1 = 2 \cdot \text{prec} \cdot \text{rec} / (\text{prec} + \text{rec})$ , and $\text{Avg. Accuracy} = \sum_{i=1}^c \frac{TP_i + TN_i}{TP_i + FN_i + FP_i + TN_i}/c$, respectively, where $TP_i, FP_i, FN_i$, and $TN_i$ denote the true positive, false positive, false negative, and true negative count for the class $i$, respectively. Note that the larger value of the all evaluation metrics above indicates the better performance of the models.

\begin{table}
\centering
	\begin{tabular}{llrrrrrrrrrrrrrrrrrr}
	\toprule
	Subspace Size& Measure&\multicolumn{6}{c}{Methods} \\
	\cmidrule(lr){1-8} 
	 & & VWG & SVM & KFA &  RIE & RRC & GLM \\
	 \cmidrule(lr){3-8} 
	\multirow{4}{*}{$n_i=10$}& Precision & \bf{0.7450} & 0.6940 & 0.7102 &  0.7098 & 0.7403 & 0.6589 \\
	& Recall & \bf{0.7490} & 0.6843 & 0.7185 &  0.7139 & 0.7282 & 0.6675\\
	& $F_1$ & \bf{0.7389} & 0.6779 & 0.7092 &  0.7006 & 0.7215 & 0.6515\\
	& Accuracy & \bf{0.9297} & 0.9110 & 0.9196 &  0.9195 & 0.9246 & 0.9091\\
	\midrule	 
	\multirow{4}{*}{$n_i=50$}& Precision & \bf{0.8243} & 0.7791 & 0.8001 & 0.7649 & 0.7894 & 0.7545 \\
	& Recall &\bf{0.8366} & 0.7899 & 0.8111 & 0.7703 & 0.7844 & 0.7698\\
	& $F_1$ & \bf{0.8271} & 0.7794 & 0.8023 & 0.7638 & 0.7823 & 0.7577\\
	& Accuracy &  \bf{0.9539} & 0.9399 & 0.9460 & 0.9362 & 0.9407 & 0.9363\\
	\midrule
	\multirow{4}{*}{$n_i=100$}& Precision & \bf{0.8509} & 0.8145 & 0.8259 & 0.7875 & 0.8074 & 0.7769 \\
	& Recall & \bf{0.8597} & 0.8242 & 0.8345 & 0.7866 & 0.7968 & 0.7930\\
	& $F_1$ & \bf{0.8506} & 0.8137 & 0.8234 &  0.7867 & 0.7935 & 0.7820\\
	& Accuracy & \bf{0.9609} & 0.9500 & 0.9527 & 0.9411 & 0.9441 & 0.9419\\
	\midrule		
	\bottomrule
	\end{tabular}\caption{Results of simulation study. The best result for each category is in bold. RIE denotes KRRC with the intrinsic (Riemannian) distance based Gaussian kernel.}
	\label{table:result}
	\end{table}

Classification results for $n_i = 10, 50$, and $100$ are summarized in \Cref{table:result}. Among the evaluation measures considered, $F_1$ is used when we need a balance between precision and recall. For this reason and for the sake of clarity, \Cref{fig:F1} only displays the boxplots of $F_1$ score for subspace sizes $n_i=\{10, 100\}$. The results presented are consistent with what we previously expected. In an intuitive sense, the unsatisfactory performance of the Euclidean distance based methods, such as RRC, and GLM, is mainly due to the fact that the nonlinear geometrical structure of the data space couldn't be taken into account in these methods. Moreover, according to the inadequate result produced by the KRRC with the intrinsic Riemannian Gaussian kernel, we would emphasize that the positive definiteness of kernel is essentially required to obtain better performance. 
\begin{figure}
	\centering
	\includegraphics[width=0.8\linewidth]{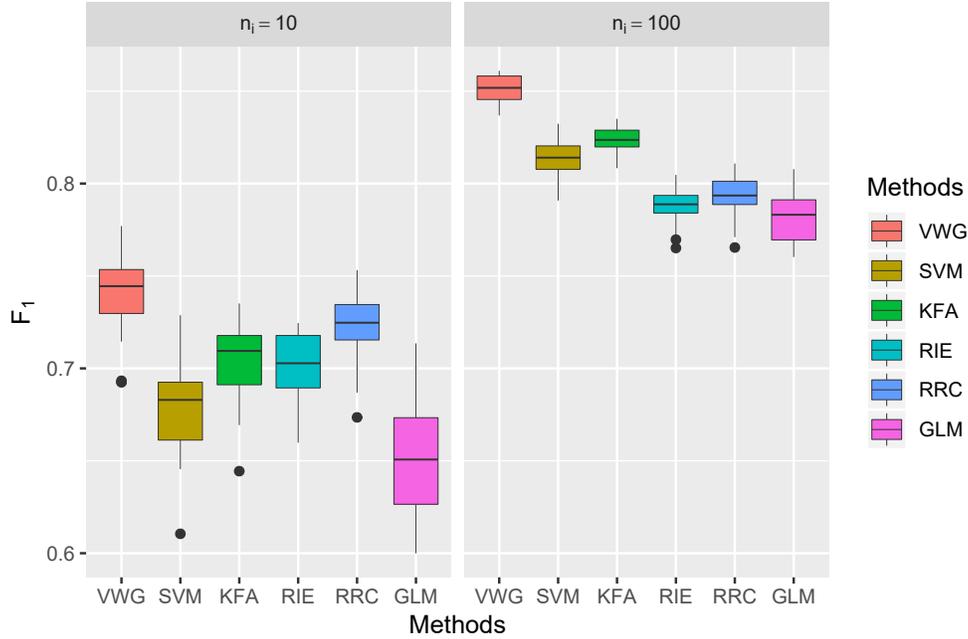}
	\caption{Results of simulation study. The left panel uses a small subspace size ($n_i=10$), and the right panel uses a large subspace size ($n_i=100$). Boxplots show the interquartile range with whiskers extending out up to 1.5 times this range, and median marked as a line in the box.}\label{fig:F1}
\end{figure} 

In addition to gain more insight into how the subspace size affects the model performance, we further investigate results by presenting \Cref{fig:classification}, in which comparisons are made using three other measures under different subspace sizes. As shown in the figure, extrinsic KRRC with VWG kernel is seen to work outstandingly well, attaining the best performance in terms of all measures. It also can be clearly observed that the performance gap between extrinsic KRRC (VWG) and other methods raised as the number of subspace size increased. Hence, besides its simpler design and implementation, extrinsic KRRC allows dealing better with large data set.

Also, the predictive ability of RKHS models such as VWG, SVM, and KFA is verified through the experimental result under a large data settting ($n_i:80 \sim 100$), on which other methods perform poorly. This indicates that, when given training data is large enough, RKHS methods are capable of capturing a nonlinear associations between the shapes of leaves and their categories.

The overriding messages of this experiment are (i) when the underlying geometry of the input space is nonlinear manifold, extrinsic KRRC (VWG) performs uniformly better than competing methods in both small and large sample sizes, and (ii) selection of an appropriate kernel is another important issue when implementing KRRC. For example, though the theoretical investigation of kernels will not be discussed here in more detail, it is worth pointing out that KRRC with intrinsic distance (Riemannian distance) based kernel shows poor performance and even worse than the naive RRC. By contrast, extrinsic VW Gaussian kernel gives a substantial boost in performance over the RRC.

\begin{figure}
	\centering
	\includegraphics[width=0.99\linewidth]{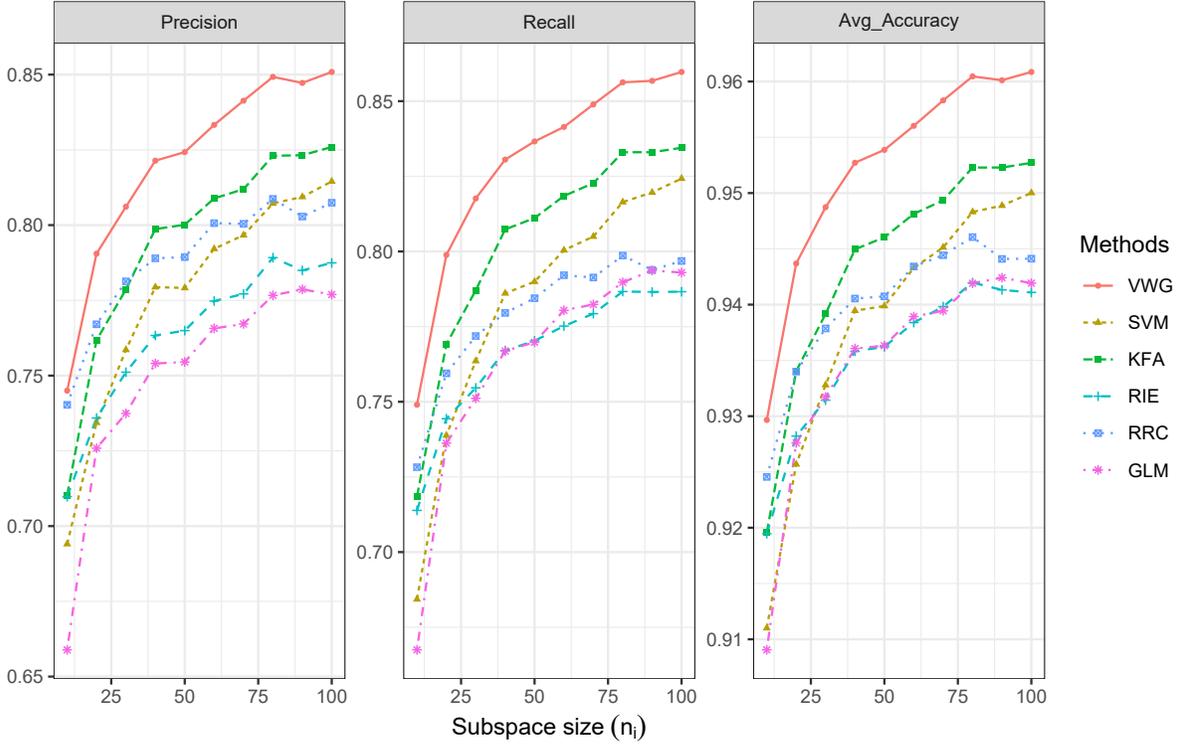}
	\caption{Plots shwow the performace of methods in terms of precision (left panel), recall(middle panel) and accuracy(rightpanel).}\label{fig:classification}
\end{figure}

%%%%%%%%%%%%%%%%%%%%%%%%%%%%%%%%%%%%%%%%%%%%%%%%%%%%%%%%%%%%%%%%%%%%%%%%%%%%%%
%% Conclusion
%%%%%%%%%%%%%%%%%%%%%%%%%%%%%%%%%%%%%%%%%%%%%%%%%%%%%%%%%%%%%%%%%%%%%%%%%%%%%%

\section{Conclusion}\label{sec:conclusion}

This paper addressed the classification problem on Kendall's planar shape space $\Sigma_2^k$ by proposing the extrinsic KRRC. As we have demonstrated throughout this paper, our approach stems from an attempt to develop a new kernel on $\Sigma_2^k$. It is desirable to employ extrinsic approach in aiming to provide a valid kernel on Riemannian manifolds, where a positive definite kernel may not be directly applicable. Simply taking the Euclidean distance induced by VW embedding, and capturing nonlinear patterns in the manifold valued data, the combination of the extrinsic approach and the kernel method not only guarantees the proposed kernel is positive definite, but also achieves a promising performance. 

We would like to conclude this paper by indicating the potential directions of future work. While in this paper the proposed extrinsic kernel has focused only on the Kendall's planar shape space, our approach can be extended in a natural way to other manifolds where well defined embedding into Euclidean space is available. Examples of such manifolds include the 3D projective shape space of $k$-ads denoted by $P\Sigma_3^k$, the space of symmetric positive definite matrices and the Grassmannian manifold. So it would still be of interest to further investigate that the extrinsic Gaussian kernel works well in the aforementioned manifolds. We eventually expect that without suffering from the non-positive definiteness problem associated with Gaussian kernels defined on manifolds, our extrinsic kernel approach will contribute to the development of new ways of kernelized methods on manifolds.
\bibliographystyle{apa}
\bibliography{KRRC_shape}

\end{document}